\documentclass[nohyperref]{article}

\usepackage{microtype}
\usepackage{graphicx}
\usepackage{subfig}
\usepackage{booktabs} 

\usepackage{hyperref}

 \usepackage[accepted]{icml2022}

\usepackage{amsmath}
\usepackage{amssymb}
\usepackage{mathtools}
\usepackage{amsthm}

\usepackage[capitalize,noabbrev]{cleveref}

\newtheorem{theorem}{Theorem}[section]
\newtheorem{lemma}[theorem]{Lemma}
\newtheorem{definition}[theorem]{Definition}
\newtheorem{assumption}[theorem]{Assumption}

\theoremstyle{plain}

\theoremstyle{definition}
\theoremstyle{remark}
\newtheorem{remark}[theorem]{Remark}

\usepackage[textsize=tiny]{todonotes}

\begin{document}

\twocolumn[
\icmltitle{Demystify Optimization and Generalization of Over-parameterized PAC-Bayesian Learning}



\icmlsetsymbol{equal}{*}

\begin{icmlauthorlist}
\icmlauthor{Wei Huang}{equal,yyy}
\icmlauthor{Chunrui Liu}{equal,yyy}
\icmlauthor{Yilan Chen}{comp}
\icmlauthor{Tianyu Liu}{yyy}
\icmlauthor{Richard Yi Da Xu}{yyy}
\end{icmlauthorlist}

\icmlaffiliation{yyy}{ University of Technology Sydney, Australia}
\icmlaffiliation{comp}{University of California San Diego
La Jolla, CA}

\icmlcorrespondingauthor{Wei Huang}{weihuang.uts@gmail.com}

\icmlkeywords{Machine Learning, ICML}

\vskip 0.3in
]



\printAffiliationsAndNotice{\icmlEqualContribution} 

\begin{abstract}
PAC-Bayesian is an analysis framework where the training error can be expressed as the weighted average of the hypotheses in the posterior distribution whilst incorporating the prior knowledge. In addition to being a pure generalization bound analysis tool, PAC-Bayesian bound can also be incorporated into an objective function to train a probabilistic neural network, making them a powerful and relevant framework that can numerically provide a tight generalization bound for supervised learning. For simplicity, we call probabilistic neural network learned using training objectives derived from PAC-Bayesian bounds as {\it PAC-Bayesian learning}. Despite their empirical success, the theoretical analysis of PAC-Bayesian learning for neural networks is rarely explored. This paper proposes a new class of convergence and generalization analysis for PAC-Bayes learning when it is used to train the over-parameterized neural networks by the gradient descent method. For a wide probabilistic neural network, we show that when PAC-Bayes learning is applied, the convergence result corresponds to solving a kernel ridge regression when the probabilistic neural tangent kernel (PNTK) is used as its kernel. Based on this finding, we further characterize the uniform PAC-Bayesian generalization bound which improves over the Rademacher complexity-based bound for non-probabilistic neural network. Finally, drawing the insight from our theoretical results, we propose a proxy measure for efficient hyperparameters selection, which is proven to be time-saving.
\end{abstract}

\section{Introduction}
Deep learning has demonstrated powerful learning capability thanks to its over-parameterization structure, in which various network architectures have been responsible for the great leap in performance \cite{lecun2015deep}. For the reason that over-fitting and complex hyperparameters tuning are two of the major pain points in deep learning, designing generalization guarantee for deep networks is a long-term pursuit for researchers \cite{zhang2021understanding}. To this end, a learning framework that trains a probabilistic neural network with a PAC-Bayesian bound objective function has been proposed \cite{begin2016pac,DBLP:journals/corr/DziugaiteR17,neyshabur2017pac,raginsky2017non,neyshabur2017exploring,london2017pac,smith2017bayesian,liu2021pac}. We dub this learning method, {\it PAC-Bayesian learning}. The PAC-Bayesian learning has been proven to be able to achieve competitive expected test set error, while providing a tight generalization bound. Furthermore, this generalization bound computed from the training error can obviate the need for splitting data into a train, test, or validation set, which is highly applicable to train a deep network with scarce data \cite{tight_pac_bay,grunwald2020fast}. Meanwhile, these advancements on PAC-Bayesian bounds have been widely adapted with different deep neural network structures including meta-learning \cite{amit2018meta}, convolutional neural network \cite{zhou2018non}, binary activated multilayer networks \cite{letarte2019dichotomize}, partially aggregated neural networks \cite{biggs2020differentiable}, and graph neural network \cite{liao2020pac}.

Due to the tremendous empirical success of PAC-Bayesian learning, there is increasing interest in understanding their theoretical properties. However, they are either restricted to a specific model variant \cite{dziugaite2018entropy,dziugaite2020role} or rely heavily on empirical exploration \cite{neyshabur2017exploring}. To our best knowledge, there is neither investigation of why the training of PAC-Bayesian learning is successful nor the reason for PAC-Bayesian bound to be tight on unseen data. Therefore, it is natural to ask: for a probabilistic neural network with an objective function derived from a PAC-Bayesian bound when gradient descent is adopted,
\begin{enumerate}
    \item[\textbf{Q1:}] How effective the training (i.e., the trainability) on training set?
    \item[\textbf{Q2:}] How tight the generalization bound (i.e., generalization) compared to those learning framework using non-probabilistic neural networks?
\end{enumerate}
The exploration of the answer can be highly non-trivial due to inherented non-convex of over-parameterization \cite{jain2017non} and additional randomness introduced by probabilistic neural network \cite{specht1990probabilistic} as well as additional challenges brought by the divergence between posterior/prior distribution pairs known as Kullback-Leibler (KL) divergence. 

In this paper, we answer the above questions by leveraging the recent advances in the theory of deep learning in over-parameterized settings with extremely (or infinitely) wide networks. It can be shown that ultra-wide networks optimized with gradient descent can achieve near-zero training error, and the critical factor that governs training process is the so called neural tangent kernel (NTK), which can be proven to be unchanged during gradient descent training \cite{jacot2018neural}, thus providing a guarantee for achieving global minimum \cite{du2019gradient,allen2019convergence}. Under the PAC-Bayesian framework, NTK is no longer calculated from the derivative of the weights directly, but instead is calculated based on the gradient of the distribution parameters of the weights. We call it Probabilistic NTK (PNTK), based on which we build a convergence analysis to characterize the optimization process of PAC-Bayes learning. Thanks to the explicit solution obtained by 
optimization analysis, we further formulate the generalization bound 
of PAC-Bayesian learning for the first time, and demonstrate its advantage by comparing with theoretical generalization bound of learning framework with non-stochastic neural networks \cite{arora2019fine,cao2019generalization,hu2019simple}. We summarize our contribution as follows:

\begin{itemize}

    \item With a detailed characterization of gradient descent training of PAC-Bayes objective function, we arrive at a conclusion that the final solution is kernel ridge regression with its kernel being the PNTK.
    \item Based on the optimization solution, we derive a PAC-Bayesian bound for deep networks and find that it improves over the Rademacher complexity bound for non-probabilistic neural network with a fair comparison.
    \item PAC-Bayesian learning depends on a large number of hyperparameters. We design a training-free proxy based on our theoretical bound and show it is effective and time-saving. 
    
    \item Technically, our proof of convergence is similar to \cite{du2018gradient,arora2019fine}. But there are essential differences. The main difference is that our network architecture is much complex (e.g. probabilistic network contain two set of parameter) and each set involves its own randomness which requires bounding many terms in this work differently and more elaborately. Due to the particularly wide range of application scenarios of probabilistic neural networks such as Variational Auto-Encoder \cite{pu2016variational} and Deep Bayesian Network \cite{nie2018deep}, we hope our technique can provide the basis for analysis of wide probabilistic neural networks.
\end{itemize}

\section{Related Work}

\paragraph{PAC-Bayesian Analysis.}  Probably Approximately Correct (PAC) Bayes framework \cite{mcallester1999pac,mcallester1999some} can incorporate knowledge about the learning algorithm and probability distribution over a set of hypotheses, thus providing a test performance (generalization) guarantee. Subsequently, PAC-Bayesian method is adopted to analyze the generalization bound of the probabilistic neural network \cite{langford2002not}. The original PAC-Bayes theory only works with bounded loss function, \citet{haddouche2021pac}  expands the PAC-Bayesian theory to learning problems with unbounded loss functions. Besides, several improved PAC-Bayesian bounds suitable for different scenarios are introduced by \cite{begin2014pac,begin2016pac}. As a result of the flexibility and generalization properties of PAC-Bayes, it is widely used to analyze the complex, non-convex, and overparameterized optimization problem, especially over-parameterized neural networks \cite{guedj2019primer}. \citet{neyshabur2017pac}  present a generalization bound for feedforward neural networks with ReLU activations in terms of the product of the spectral norm of the layers and the Frobenius norm of the weights. \citet{london2017pac}  study the generalization error of randomized learning algorithms -- focusing on stochastic gradient descent (SGD) -- using a novel combination of PAC-Bayes and algorithmic stability.



\paragraph{PAC-Bayesian Learning.} 
Apart from providing theoretical analysis for the generalization properties of deep learning, PAC-Bayes pay an increasing number of attention to afford numerical generalization bound (certificate) for practical deep learning algorithm. \citet{NIPS2001_98c72428} unprecedentedly introduce a method to train a Bayesian neural network and use a refined PAC-Bayesian bound for computing error upper bound. Later, \citet{neyshabur2017exploring} extend \citet{NIPS2001_98c72428}'s work by developing a training objective function derived from a relaxed PAC-Bayesian bound. In the standard application of PAC-Bayes, the prior is typically chosen to be a spherical Gaussian centered at the origin. However, without incorporating the information of data, the KL divergence might be unreasonable large, limiting the performance of PAC-Bayes method. To fill this gap, a large literature
propose to obtain localized PAC-Bayes bounds via distribution-dependent priors through data \cite{ambroladze2007tighter, negrea2019information, dziugaite2020role, perez2021learning}. Furthermore, \citet{dziugaite2018data} show how an
differentially private data-dependent prior yields a valid PAC-Bayes bound for the situation that data distribution is presumed to be unknown. More recently, research focus on providing PAC-Bayesian bound for more realistic architectures, such as, ImageNet \cite{zhou2018non}, binary activated multilayer networks \cite{letarte2019dichotomize}, Partially Aggregated Neural Networks \cite{biggs2020differentiable}, and graph neural network \cite{liao2020pac}. We denote the practical use of PAC-Bayesian algorithm to train over-parameterized neural networks as PAC-Bayesian learning and the target of this work is to demystify the success behind the deep learning trained via PAC-Bayesian bound. 


\paragraph{Neural Tangent Kernel.} One of the greatest recent progress in deep learning theory is the gradient descent training can find global minimum with deep neural networks thanks to the over-parameterization \cite{du2019gradient,du2018gradient,allen2019convergence}. Technically, \cite{jacot2018neural} resort to the NTK, which will stay constant during training. Under a mild assumption that the lowest eigenvalue of the NTK is greater than zero, the loss that is convex regards output will convergence to the global minimum. Furthermore, the generalization ability to unseen data of trained ultra-wide networks has been characterized by the Rademacher complexity analysis \cite{cao2019generalization,arora2019fine}. In this work, we will make a comparison between Rademacher bound and PAC-Bayesian bound with ultra-wide networks, where Rademacher bound of kernel ridge regression is formulated by \citep{hu2019simple}. Since the origin of the NTK, it has been applied to different deep networks structures, such as orthogonal initialization \cite{huang2020neural}, convolutional networks \citep{arora2019exact}, graph neural networks \citep{du2019graph,huang2021wide}, and transformer \cite{hron2020infinite}.

\section{Preliminary}

\subsection{Notation}
We use bold-faced letters for vectors and matrices otherwise representing scalar. We use $\| \cdot \|_2$ to denote the Euclidean norm of a vector or the spectral norm of a matrix, while denoting $\| \cdot \|_F$ as the Frobenius norm of a matrix. For a neural network, we denote $\phi(x)$ as the activation function and we adopt ReLU activation in this work. Given a logical variable $a$, then $\mathbb{I}(a) = 1$ if $a$ is true otherwise  $\mathbb{I}(a) = 0$. Let ${\bf I}_d$ be the identity matrix with dimension of $\mathbb{R}^{d \times d}$. We denote $[n] = \{1,2,\ldots,n \}$. The least eigenvalue of matrix $\bf A$ is denoted as $\lambda_0({\bf A}) = \lambda_{\min}(\bf A)$.

\subsection{Over-parameterized Probabilistic Neural Network}

In PAC-Bayesian learning we use probabilistic neural networks instead of deterministic networks, where the weights always follow a certain distribution. In this work, we adopt the Gaussian distribution for the weights, and define a two-layer probabilistic neural network governed by the following expression,

\begin{equation}\label{eq:network}
    f( {\bf x}) = \frac{1}{\sqrt{m}} {\bf v} \phi({\bf Wx})
\end{equation}
where ${\bf x} \in \mathbb{R}^d$ is the input, ${\bf W} \in \mathbb{R}^{m \times d}$ are weight matrix in the first layer, and $ {\bf v} \in \mathbb{R}^m$ is the weight vector in the second layer. For simplicity, we represent the weight matrix as a set of column vectors, i.e., ${\bf W} = ({\bf w}_1, {\bf w}_2, \dots, {\bf w}_m )^\top$. In particular, we use the re-parametrization trick to treat the weight in the first layer and leave the weights in the second to be initialized by uniform distribution,

\begin{equation}\label{eq:weight}
\begin{aligned}
& {\bf w}_r = \boldsymbol{\mu}_r + \boldsymbol{\sigma}_r  \odot \boldsymbol{\epsilon}_r, ~~~ \boldsymbol{\epsilon}_r \sim \mathcal{N}({\bf 0},{\bf I}_d); \\ &~~~ {\bf v}_r \sim {\rm unif}(\{-1,1\})
\end{aligned}
\end{equation}

where $\odot$ denotes the element-wide product operation, $\boldsymbol{\mu}$ and $\boldsymbol{\sigma}$ are mean and variance respectively. We then fix the second layer ${\bf v}$ and only optimize the first layer. This is a trick that has been widely used in the NTK paper \cite{arora2019fine,du2018gradient} for simplifying analysis, the same results can be obtained on deep probabilistic neural network. 

\subsection{The PAC-Bayesian Learning}

Suppose data $S = \{ ({\bf x}_i,y_i)\}_{i=1}^n$ are i.i.d. samples from a non-degenerate distribution $\mathcal{D}$. Defining that the $\mathcal{H}$ represents the hypothesis space, $h(\bf{x})$ is the prediction of hypothesis $ h \in \mathcal{H} $ over sample $\bf{x}$. $ R_\mathcal{D}(h) = \mathbb{E}_{ ({\bf x}, y) \sim \mathcal{D} } [ \ell (y, h(\bf{x})) ] $ represents the generalization error of classifier $h$ and $ \widehat  R_S (h) = \frac{1}{n} \sum_{i=1}^{n} \ell \left( y_{i}, h \left( {\bf x}_{i} \right) \right)$ represents the empirical error of classifier $h$, where $\ell(\cdot)$ is the loss function.

The PAC-Bayes theorem \cite{langford2001bounds,seeger2002pac,maurer2004note} concludes that,

\begin{theorem} \label{thm_mca}
 Denote $ Q_0 \in \mathcal{H} $ be some prior distribution over $\mathcal{H}.$ Then for any $\delta \in(0,1]$, the following inequality holds uniformly for all posteriors distributions $ Q \in \mathcal{H} $ with probability at least $ 1-\delta $,

\begin{equation}
{\rm kl}(\widehat R_S(Q) \| {R}_\mathcal{D} (Q) ) \leq \frac{ \mathrm{KL} (Q \| Q_0) + \log \frac{2\sqrt{n}}{\delta}} {n}.
\end{equation}
\end{theorem}

Here $ R_\mathcal{D}(Q) =\mathbb{E}_{({\bf x}, y) \sim \mathcal{D}, h \sim Q}[\ell(y, h({\bf x}))] =  \mathbb{E}_{ h \sim Q } [R_\mathcal{D}(h)] $ and $ \widehat R_S(Q) =  \mathbb{E}_{ h \sim Q } [\widehat R_S(h)] $.  $\mathrm{KL} (Q \| Q_0) = \mathbb{E}_{Q} \left[ \ln \frac{Q}{Q_0} \right] $ is the Kullback-Leibler (KL) divergence and  ${\rm kl}(q \| q')= q \log(\frac{q}{q'})+(1-q)\log(\frac{1-q}{1-q'})$ is the binary KL divergence. Furthermore, combining with pinsker's inequality for binary KL divergence,  ${\rm kl}(\hat p \| p) \ge (p-\hat p)^2/(2p) $, when $\hat p < p$, yields,

\begin{equation} \label{eq:bound_sqrt}
R_\mathcal{D}(Q) - \widehat R_S(Q)  \leq \sqrt{ 2R_\mathcal{D}(Q) \frac{ \mathrm{KL} (Q \| Q_0) + \log \frac{2\sqrt{n}}{\delta}} {n}}.
\end{equation}

The above bound is the classical bound. On the other hand,  combining with the inequality $\sqrt{ab} \le \frac{1}{2}(\bar \lambda a +\frac{b}{ \bar \lambda})$, for all $\bar \lambda >0$, lead to a PAC-Bayes-$\lambda$ bound, which is proposed by \citep{thiemann2017strongly},

\begin{theorem} \label{thm_lambda} Denote $ Q_0 \in \mathcal{H} $ be some prior distribution over $\mathcal{H}.$ Then for any $\delta \in(0,1]$, the following inequality holds uniformly for all posteriors distributions $ Q \in \mathcal{H} $ with probability at least $ 1-\delta $

\begin{equation} \label{eq:lambda}
{R}_\mathcal{D} (Q) \leq \frac{\widehat{R}_S(Q)}{1- \bar \lambda/2} +  \frac{ \mathrm{KL} (Q \| Q_0) + \log \frac{2\sqrt{n}}{\delta}} {n \bar \lambda(1- \bar \lambda/2)}.
\end{equation}

\end{theorem}

In this work, we study the PAC-Bayesian learning by training the probabilistic network described by Equation (\ref{eq:network}) with an objective  PAC-Bayes bound (\ref{eq:lambda}). Specifically, the objective function can be expressed as,

\begin{equation} \label{eq:target}
\mathcal{B}(Q) = \widehat{R}_{S}(Q) + \lambda \frac{\mathrm{KL}(Q \|Q_0) }{n}
\end{equation}
where $\lambda$ is a hyperparameter introduced in a heuristic manner to make the method more flexible, and $\bar{\lambda} = 2$. Since the term regarding $\delta$ is a constant, we omit it in the objective function. Instead of optimizing ${\bf W}$ itself directly, we introduce reparameterization trick \cite{kingma2015variational} to perform gradient descent,
\begin{equation} \label{eq:repara}
\begin{aligned}
    {\boldsymbol \mu}_{r}(t+1) =  {\boldsymbol \mu}_{r}(t) - \eta \frac{\partial
    \mathcal{B}}{ \partial  {\boldsymbol \mu}_r(t)}  \\
      {\boldsymbol \sigma}_r(t+1) = {\boldsymbol \sigma}_r(t) - \eta \frac{\partial
    \mathcal{B}}{ \partial {\boldsymbol \sigma}_r(t)}
    \end{aligned}
\end{equation}
where $\eta$ is the learning rate.

\subsection{Neural Tangent Kernel}

The neural tangent kernel was originated from \citep{jacot2018neural}, whose definition is given by,
\begin{equation}
    \boldsymbol{\Theta}({\bf X},{\bf X},t) = \nabla_{ \boldsymbol{\theta} } f ( {\bf X},t ) \nabla_{ \boldsymbol {\theta} } f ( {\bf X},t )^\top 
\end{equation}
where $\boldsymbol {\theta} \in \mathbb{R}^{p \times 1} $ with $p$ being the number of parameters in the model. We demonstrate how it is used to describe the training dynamics of deep neural networks.  

The dynamics of gradient flow for parameters are given by,
\begin{equation}
\frac{\partial \boldsymbol{\theta}}{\partial t} = - \eta \nabla_{\boldsymbol{\theta}} \widehat{R}_S = -\eta \nabla_{\boldsymbol{\theta}} f({\bf X},t)^\top \nabla_{f({\bf X},t)} \widehat{R}_S
\end{equation}
Then the dynamics of output functions $f({\bf X}) = {\rm vec}(f({\bf x})_{{\bf x} \in {\bf X}}) \in \mathbb{R}^{nm \times 1}$ follow,
\begin{equation}
\frac{\partial f({\bf X},t)}{\partial t} =  \nabla_{\boldsymbol{\theta}} f({\bf X},t) \frac{\partial \boldsymbol{\theta}}{\partial t} = -\eta \boldsymbol{\Theta}({\bf X},{\bf X},t) \nabla_{f({\bf X},t)} \widehat{R}_S
\end{equation}
It is proven that when the width of neural networks is infinite, { $\boldsymbol{\Theta}({\bf X},{\bf X},t)$ converge in probability to $\boldsymbol{\Theta}({\bf X},{\bf X})$ } \cite{jacot2018neural}, i.e., $\boldsymbol{\Theta}({\bf X},{\bf X},t) = \boldsymbol{\Theta}({\bf X},{\bf X}) $.  Assuming that the loss is the mean squared error loss, $\widehat{R}_S = \frac{1}{2}\|f({\bf X})-{\bf Y}) \|^2_2$, then we have,
\begin{equation}
\begin{aligned}
  f_t(\textbf{x}) = & f_0(\textbf{x}) + \boldsymbol{\Theta}(\textbf{x},\textbf{X})^T \boldsymbol{\Theta}^{-1} (\textbf{I} + e^{-\eta \boldsymbol{\Theta} t})  (\textbf{Y}-f_0(\textbf{X}))
  \end{aligned}
\end{equation}
At the end of the training, namely, $t \rightarrow \infty$, the solution is equivalent to the well-known kernel regression, of which the kernel is NTK.

\section{Theoretical Results}

\subsection{Optimization analysis} \label{sec:4.1}

To simplify the analysis, we first consider the optimization of probabilistic neural networks of the form (\ref{eq:network}) with objective $\widehat{R}_S(Q)$ and adopting the mean squared error loss.
In other words, we neglect the KL divergence term at this stage and will show that the corresponding results can be extended to the target function with KL divergence. To perform the optimization analysis, we define the neural tangent kernel for probabilistic neural network, and name them as {\it probabilistic neural tangent kernel} (PNTK).
\begin{definition}[Probabilistic Neural Tangent Kernel] \label{eq:ntk} 
The tangent kernels associated with output function $f({\bf x},t)$ at parameters $\boldsymbol{\mu}$ and $\boldsymbol{\sigma}$ are defined as, 
\begin{equation}\label{eq:NTK_kernel}
\begin{aligned}
\boldsymbol{\Theta}^{\mu}({\bf X},{\bf X},t) &=\nabla_{ \boldsymbol{\mu} }  f ({\bf X},t) \nabla_{ \boldsymbol{\mu} }  f({\bf X},t )^{\top} \in \mathbb{R}^{n \times n} \\
\boldsymbol{\Theta}^{\sigma}({\bf X},{\bf X},t) &=\nabla_{ \boldsymbol{\sigma} }  f({\bf X},t) \nabla_{ \boldsymbol{\sigma} }  f({\bf X},t)^{\top} \in \mathbb{R}^{n \times n} \\
\end{aligned}
\end{equation}
\end{definition}

Different from standard (deterministic) neural network, the probabilistic network consist of two sets of parameters, i.e., $\boldsymbol{\mu}_{r}$ and $\boldsymbol{\sigma}_{r}$, with $r=[m]$. The gradient descent based on reparameteriztion trick (Equation \ref{eq:repara}) is performed on each of the two parameter set separately. As a result, we have two corresponding tangent kernels. It is shown that the NTK will converge in probability to a deterministic kernel in the infinite-width limit \cite{jacot2018neural}, as we will show in later derivation, the PNTK will converge to a limiting kernel at initialization and during training with a large width condition. The detailed expression for the limiting kernel is given as follows

\begin{definition}
[Limiting Neural Tangent Kernel with ReLU Activation] \label{lem:pntk}
Consider a probabilistic network of the form (\ref{eq:network}), with a Gaussian distribute weights ${\bf w} $. Then the tangent kernels in the infinite-width limit will follow the expression,
\begin{equation} \label{eq:limit_ntk}
\boldsymbol{\Theta}^\infty_{ij}({\bf X},{\bf X}) = \mathbb{E}_{{\bf w}  } \big[{\bf x}^\top_i {\bf x}_j \mathbb{I} \{ {\bf w}^\top {\bf x}_i \ge 0, {\bf w}^\top {\bf x}_j \ge 0 \} \big]
\end{equation}
\end{definition} 

The definition above gives the exact formulation of the limiting NTK for random networks. In addition, its positive definiteness is a critical factor to guarantee the convergence of gradient descent, thus we state the assumption as follows, 

\begin{assumption}
Assume the lowerest eigenvalue of the limiting NTK is greater than zero, i.e., $\lambda_0(\boldsymbol{\Theta}^\infty) > 0$.
\end{assumption}

With all the preliminaries above, we state our main theorem,

\begin{theorem}[Convergence of Probabilistic Networks with large width] \label{thm:opt}
Suppose the network's width is of $m = \Omega \big(\frac{n^6 }{\lambda_0^4 \delta^3(\sigma_0^2 +\hat \sigma^2)} \big)$ with the initialization. Then with probability at least $1- \delta$ over the initialization, we have,
\begin{equation}
 \widehat{R}_S(Q,t)  \le \exp(-\lambda_0 t) \widehat{R}_S(Q,0) 
\end{equation}
where we adopt initialization of  ${\mu}_{ri} \sim \mathcal{N}(0, {\hat \sigma})$ and  ${\sigma}_{ri} = \sigma_0$, for $i \in [d]$. ${\mu}_{ri}$ and ${\sigma}_{ri}$ are entries of weights matrix $\boldsymbol{\mu}$ and  $\boldsymbol{\sigma}$, $\hat \sigma$ and $\sigma_0$ are positive constants. 
\end{theorem}


Our theorem establishes that if $m$ is large enough, the expected training error converges to zero at a linear rate. In particular, the least eigenvalue of PNTK governs the convergence rate. To prove the theorem, we split the process into three parts, namely, PNTK at initialization, PNTK during training and changes of weights during training, and all the detailed proofs for our theorems and lemmas can be found in the Appendix.

We first study the behavior of tangent kernels with ultra-wide condition, namely $m = {\rm ploy}(n, 1/\lambda_0,1/\delta)$, at initialization. The following lemma demonstrates that if $m$ is large, then $\boldsymbol{\Theta}^\mu_0$ and  $\boldsymbol{\Theta}^\sigma_0$ have a lower bound on smallest eigenvalue with high probability. The proof is by the standard concentration bound.  

\begin{lemma} [PNTK at initialization]
If $m =O \big(\frac{n^2\log(n^2/\delta)}{\lambda^2_0} \big)$, while $\boldsymbol{\mu}_r$ and $\boldsymbol{\sigma}_r$ are initialized by the form in Theorem \ref{thm:opt}, then with probability at least $1-\delta$, $\big \|\boldsymbol{\Theta}^\mu({\bf X}, {\bf X},0) -\boldsymbol{\Theta}^\infty({\bf X}, {\bf X}) \big \|_2  \le \frac{\lambda_0}{4}$, $\big \|\boldsymbol{\Theta}^\mu({\bf X}, {\bf X},0) \big \|_2  \ge \frac{3\lambda_0}{4}$;
and $\big \|\boldsymbol{\Theta}^\sigma({\bf X}, {\bf X},0) -\boldsymbol{\Theta}^\infty({\bf X}, {\bf X}) \big \|_2   \le \frac{\lambda_0}{4}$, $\big \|\boldsymbol{\Theta}^\sigma({\bf X}, {\bf X},0) \big \|_2  \ge \frac{3\lambda_0}{4}$.
\end{lemma}

\begin{remark}
The concentration bound is over randomness of initialization for $\boldsymbol{\mu}_{r}(0)$ and randomness of variance $\boldsymbol{\epsilon}_r$. 
\end{remark}

From the above lemma, we implicitly show that both tangent kernels regarding $\boldsymbol{\mu}_r$ and $\boldsymbol{\sigma}_r$ are equivalent in the infinite-width limit. It is also true in general as long as the reparametrization trick is adopted for training probabilistic neural networks. 
The next problem is that PNTKs are time-dependent matrices, thus varying during gradient descent training. To account for this problem, we build a lemma stating that if the weights ${\bf w}(t)$ is close to ${\bf w}(0)$ during gradient descent training, then the corresponding PNTKs $\boldsymbol{\Theta}^{\mu}(t)$ and $\boldsymbol{\Theta}^{\sigma}(t)$ are close to $\boldsymbol{\Theta}^\infty$ and have a lower bound over smallest eigenvalue $\lambda_0$.
Here we introduce ${\bf w}_{r}(t) \equiv  \boldsymbol{\mu}_{r}(t) + \boldsymbol{\sigma}_{r}(t)  \odot \boldsymbol{\epsilon}_{r}(0)$, where $\boldsymbol{\epsilon}_{r}(0) $ copies sample value  from that in ${\bf w}_{r}(0)$. Notice ${\bf w}_{r}(t)$ in our definition is not exact weights at training step $t$, but we introduce it to facilitate our proof.

\begin{lemma} [PNTK during training] \label{lem:train}
If $\boldsymbol{\mu}_r$ and $\boldsymbol{\sigma}_r$ are initialized the same with Theorem \ref{thm:opt}, for any set of weight vectors ${\bf w}_1,\ldots, {\bf w}_m $ that satisfy for any $r \in [m]$, $\| {\bf w}_{r}(t) - {\bf w}_{r}(0) \| \le  \frac{c \lambda_0 \delta \sqrt{\sigma_0^2+ \hat \sigma^2}}{n^2} \equiv R $, where $c$ is a constant, then with probability at least $1-\delta$, $\big \|\boldsymbol{\Theta}^\mu({\bf X}, {\bf X},t) -\boldsymbol{\Theta}^\infty({\bf X}, {\bf X}) \big \|_2  \le \frac{\lambda_0}{4}$, $\big \|\boldsymbol{\Theta}^\mu({\bf X}, {\bf X},t) \big \|_2  \ge \frac{\lambda_0}{2}$;
and $\big \|\boldsymbol{\Theta}^\sigma({\bf X}, {\bf X},t) -\boldsymbol{\Theta}^\infty({\bf X}, {\bf X}) \big \|_2   \le \frac{\lambda_0}{4}$, $\big \|\boldsymbol{\Theta}^\sigma({\bf X}, {\bf X},t) \big \|_2  \ge \frac{\lambda_0}{2}$.
\end{lemma}


By the above lemma, we demonstrate that if the change of weight is bounded, then the tangent kernel matrix is closed to its expectation. The next lemma will show that the change of weight is bounded during gradient descent training when the PNTK is close to the limiting kernel Equation (\ref{eq:limit_ntk}).

\begin{lemma} [Change of weights during training]\label{lem:duringtrain}
Suppose $\lambda_{\min}(\boldsymbol{\Theta}(t)) \ge \frac{\lambda_0}{2}$ for $0 < t < T $. Then we have $  \widehat{R}_S(Q,t) \le \exp(-{\lambda_0} t)\widehat{R}_S(Q,0) $, and with probability at least $1-\delta$ over initialization, $\| {\bf w}_{r}(t) - {\bf w}_{r}(0) \|_2 \le \frac{2n}{\sqrt{m \delta}\lambda_0} \equiv R' $.
\end{lemma}

The main gap between PNTK and final results for output function can be filled by writing down the gradient flow dynamics,
\begin{equation}
\begin{aligned}
& \frac{d  f({\bf x}_i,t)}{d t}  = \sum_{j=1}^n ({\bf y}_j-  f({\bf x}_j,t)) (\boldsymbol{\Theta}^{\mu}_{ij} + \boldsymbol{\Theta}^{\sigma}_{ij})
\end{aligned}
\end{equation}
Combined with fact that PNTKs are bounded during training, therefore the behavior of the loss is traceable. To finalize the proof for Theorem \ref{thm:opt}, we construct a final lemma.

\begin{lemma}\label{lem:4.9}
If $R' < R$, we have
$\widehat{R}_S(Q,t)  \le \exp(-\lambda_0 t) \widehat{R}_S(Q,0). 
$
\end{lemma}

\subsection{Training with KL divergence}\label{sec:4.2}

According to Equation (\ref{eq:target}), there is a KL divergence term in the objective function. We expand the KL-divergence for two Gaussian distribution, $p({w}({t})) = \mathcal{N}({\mu}_{t}, {\sigma}_{t})$, and $p({ w}({0})) = \mathcal{N}({\mu}_{0}, {\sigma}_{0})$,
\begin{equation}
\text{KL} = \frac{1}{2} \bigg (\log \frac{{\sigma}_{0}}{{\sigma}_{t}} + \frac{({\mu}_{t}-{\mu}_{0})^2}{{\sigma}^2_{0}}+\frac{\sigma_t}{\sigma_0}-1    \bigg)
\end{equation}




We compare the gradient of mean weights $\boldsymbol{\mu}$ and variance weights $\boldsymbol{\sigma}$. With a direct calculation, we have $\frac{\partial  f({\bf x}_i)}{\partial \boldsymbol{\mu}_r} = \frac{1}{\sqrt{m}} {v}_r  \mathbb{I} ( {\bf w}^\top_r {\bf x}_i \ge 0) {\bf x}_i $ and $\frac{\partial f({\bf x}_i)}{\partial \boldsymbol{\sigma}_r} = \frac{1}{\sqrt{m}} {v}_r  \mathbb{I} ( {\bf w}^\top_r {\bf x}_i \ge 0) {\bf x}_i  \odot \boldsymbol{\epsilon}_r$. It is shown that there is one more random variable $\boldsymbol{\epsilon}_r$ associated with the gradient regarding variance weights, which leads to the expected gradient norm to be zero. To verify this result, we conduct an experiment and leave the results in the {Appendix \ref{sec:norm}}. Therefore, it is  approximately equivalent to fix $\boldsymbol{\sigma}$ during gradient descent training. With this assumption we arrive at the conclusion:

\begin{theorem} \label{thm:solution}
Consider gradient descent on target function (\ref{eq:target}), with the initialization stated in Theorem \ref{thm:opt}. Suppose $m \ge {\rm poly}({n},1/\lambda_0, 1/\delta,1/ \mathcal{E})$. Then with probability at least $1-\delta$ over the random initialization, we have
\begin{equation}
 {\widehat{f}({\bf x},\infty)} = \boldsymbol{\Theta}^\infty({\bf x},{\bf X}) \big (\boldsymbol{\Theta}({\bf X},{\bf X}) + \lambda/{\sigma^2_0} {\bf I} \big)^{-1} {\bf Y} \pm \mathcal{E}
\end{equation}
where $\widehat{f}({\bf x},t) = \mathbb{E}_{f \sim Q} f({\bf x},t) $ aligns with the definition of empirical loss function. 
\end{theorem}

The theorem above reveals that the regularization effect of the KL term in PAC-Bayesian learning, and presents an explicit expression for the convergence result of the output function. 


\subsection{Generalization analysis}\label{sec:4.3}

Recall that in Theorem \ref{thm_lambda}, the PAC-Bayesian bound with respect to the distribution at initialization and after optimization is given. Therefore, combined with results from Theorem \ref{thm:solution}, we can directly provide a generalization bound for PAC-Bayesian learning with ultra-wide networks. 

\begin{theorem} [PAC-Bayesian bound with NTK] \label{thm:PAC}
Suppose data $S = \{ ({\bf x}_i,y_i)\}_{i=1}^n$ are i.i.d. samples from a non-degenerate distribution $\mathcal{D}$, and $m \ge {\rm poly}(n, \lambda_0^{-1}, \delta^{-1})$. Consider any loss function $\ell: \mathbb{R} \times \mathbb{R} \rightarrow [0,1]$. Then with probability at least $1-\delta$ over the random initialization and the training samples, the probabilistic network trained by gradient descent for $T \ge \Omega(\frac{1}{\eta \lambda_0} \log \frac{n}{\delta})$ iterations has population risk $R_{\mathcal{D}}(Q)$ that is bounded as follows:
\small{
\begin{equation}
\begin{aligned}
   R_{\mathcal{D}} & \le \frac{  {\bf Y}^\top (\boldsymbol{\Theta}+ {\lambda}/{\sigma^2_0} {\bf I})^{-1} {\bf Y}}{n\sigma^2_0} + \frac{ \lambda}{\sigma^2_0}  \sqrt{ \frac{ {\bf Y}^\top(\boldsymbol{\Theta}+  {\lambda}/{\sigma^2_0} {\bf I})^{-2} {\bf Y}} {n}} \\ & + O \bigg(\frac{\log \frac{2\sqrt{n}}{\delta}}{n} \bigg).
   \end{aligned}
\end{equation}}
\end{theorem}

In this theorem, we establish a reasonable generalization bound for PAC-Bayesian learning framework, thus providing a theoretical guarantee. We further demonstrate the advantage of PAC-Bayesian learning by comparing it with the Rademacher complexity-based generalization bound for deterministic neural network with kernel ridge solution.  

\begin{theorem} [Rademacher bound with NTK] \label{thm:Rademacher}
Suppose data $S = \{ ({\bf x}_i,y_i)\}_{i=1}^n$ are i.i.d. samples from a non-degenerate distribution $\mathcal{D}$, and $m \ge {\rm poly}(n, \lambda_0^{-1}, \delta^{-1})$. Consider any loss function $\ell: \mathbb{R} \times \mathbb{R} \rightarrow [0,1]$. Then with probability at least $1-\delta$ over the random initialization and training samples, the network trained by gradient descent for $T \ge \Omega(\frac{1}{\eta \lambda_0} \log \frac{n}{\delta})$ iterations has population risk $R_{\mathcal{D}}(Q)$ that is bounded as follows:
\small{
\begin{equation}
\begin{aligned}
   R_{\mathcal{D}} & \le  \sqrt{\frac{  {\bf Y}^\top (\boldsymbol{\Theta}+ {\lambda}/{\sigma^2_0} {\bf I})^{-1} {\bf Y}}{n}} + \frac{\lambda}{\sigma^2_0}  \sqrt{ \frac{ {\bf Y}^\top(\boldsymbol{\Theta}+  {\lambda}/{\sigma^2_0} {\bf I})^{-2} {\bf Y}} {n}} \\ & +O \bigg(\sqrt{\frac{\log \frac{n}{\lambda_0 \delta}}{n}} \bigg).
   \end{aligned}
\end{equation}}
\end{theorem}

Theorem \ref{thm:Rademacher} is obtained by following Theaorem 5.1 in \cite{hu2019simple}, who presents a Rademacher complexity-based generalization bound for ultra-wide neural networks with kernel ridge regression solution. Similar analysis for kernel regression with NTK can be found in \cite{arora2019fine,cao2019generalization}.

The main difference between two generalization bound is $\frac{  {\bf Y}^\top (\boldsymbol{\Theta}+ {\lambda}/{\sigma^2_0} {\bf I})^{-1} {\bf Y}}{n}$ versus $\sqrt{\frac{  {\bf Y}^\top (\boldsymbol{\Theta}+ {\lambda}/{\sigma^2_0} {\bf I})^{-1} {\bf Y}}{n}}$, which is due to that PAC-Bayesian bound count the KL divergence while Rademacher bound calculate the reproducing kernel Hilbert space (RKHS) norm. We find that the rates of convergence are different. One is $O(1/n)$ and another is $O(1/\sqrt{n})$. Therefore, we conclude that the PAC-Bayesian bound has an improvement over Rademacher complexity-based bound.

\begin{figure}[t!]
\centering
\subfloat{\includegraphics[width=0.25\textwidth]{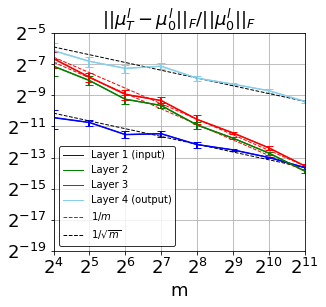}} 
\subfloat{\includegraphics[width=0.25 \textwidth]{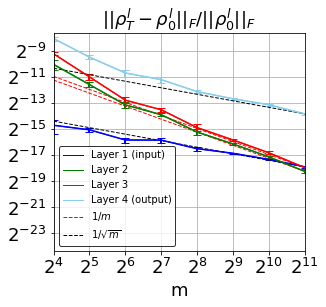}}
\vspace*{-4mm}
\caption{Relative Frobenius norm change in $\mu$ and $\rho$ respectively during training with MSE loss that derived from the classic PAC-Bayesian bound, where $m$ is the width of the network.} 
\label{fig:1} 
\end{figure}

\section{Experiments}
In this section, we first provide empirical support showing that the training dynamics of wide probabilistic neural networks using the training objective derived from a PAC-Bayes bound are captured by PNTK, which validates our Lemma \ref{lem:duringtrain}. Second, as an extension of our finding of PAC-Bayesian bound in Theorem \ref{thm:PAC}, we provide a training-free metric to approximate the PAC-Bayesian bound via NTK, which can be used to select the best hyperparameters without involving any training and eliminate excessive computation time.

\subsection{Experimental setup}
In all experiments, the NTK parameterization is chosen for initializing parameters, which follows Equation (\ref{eq:network}). Specifically, the initial mean for weights, $\mu_0$, is sampled from a truncated Gaussian distribution with a mean of 0 and a standard variance of $1$, truncating at 2 standard deviations. To make sure that the variance is positive, the initial variance for weight is transformed from the given value of $\rho_0$ through formula $\sigma_0 = \log(1 + \exp(\rho_0))$.


For the experiment in section \ref{tdwp}, we consider a three hidden layer ReLU fully-connected network of the training objective derived from the PAC-Bayesian lambda bound in Equation (\ref{eq:lambda}), used an ordinary MSE function as loss, trained with full-batch gradient descent using learning rates equal to 1 on a fixed subset of MNIST ($|D|=128$) of ten-classification. An random initialized prior with no connection to data is used since it is inline with our theoretical setting and we only intend to observe the change in parameters rather than the performance of the actual bound. 

In section \ref{PFHS}, we use both fully connected and convoluted neural network structures to perform experiments on MINIST and CIFAR-10 datasets for demonstrating the effectiveness of our training-free PAC-Bayesian network bound for searching hyperparameters under different datasets and network structures. We build a 3 layers fully-connected neural network with 600 neurons on each layer.  Another convolutional architecture with a total of 13 layers with around 10 million learnable parameters is constructed. We adopt data-dependent prior since it is a practical and popular method \cite{dziugaite2020role,perez2021learning}.  Specifically, this data-dependent prior is pre-trained on a subset of total training data with empirical risk minimization. The networks for posterior training is then initialized by the weights learnt from prior.  At last, the generalization bound is computed using Equation (\ref{eq:lambda}). Bound computing-related settings are referred to the work done by \citet{tight_pac_bay}, such as, confidence parameters for the risk certificate and chernoff bound, and the number of monte carlo samples for estimating the risk certificate.

 \subsection{Validation of theoretical results} \label{tdwp}
 
After $T=2^{17}$ steps of gradient descent updates from different random initialization, we plot the changes of (input/output/hidden layer) $\boldsymbol{\mu}$ and $\boldsymbol \rho$ for weights corresponded to the variation in width $m$ for each layer on Figure \ref{fig:1}. We observe that the relative Frobenius norm change in input/output layer's weights scales as $1/\sqrt{m}$ while hidden layers' weights scales as $1/m$ during the training, which verifies our lemma \ref{lem:duringtrain}.

 \begin{figure*}[t!]
\centering
 \subfloat{\includegraphics[width=0.28\textwidth]{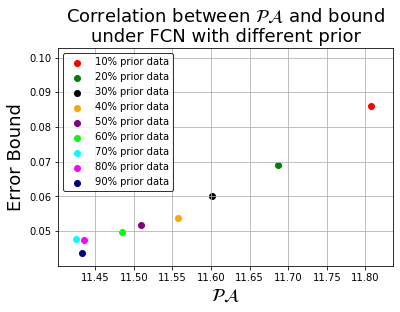}} 
 \subfloat{\includegraphics[width=0.28\textwidth]{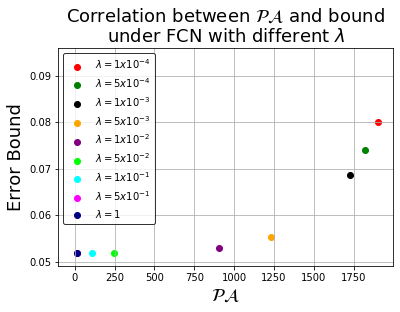}} 
\subfloat{\includegraphics[width=0.28\textwidth]{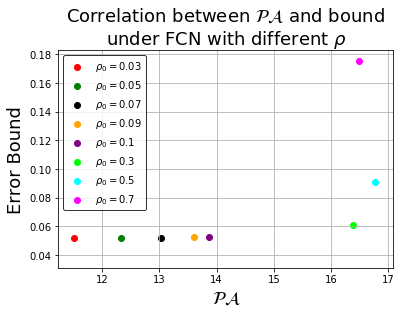}}
\vspace*{-4mm}
 \subfloat{\includegraphics[width=0.28\textwidth]{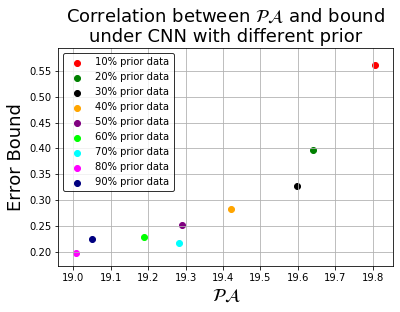}} 
 \subfloat{\includegraphics[width=0.28\textwidth]{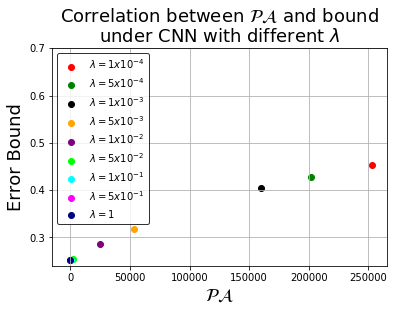}} 
\subfloat{\includegraphics[width=0.28\textwidth]{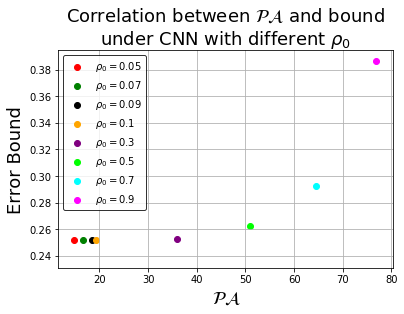}}
\vspace*{-4mm}
\caption{In FCN structure with MNIST dataset, Kendall-tau correlations between $\mathcal{PA}$ and generalization bound with respect to the different proportion of data used for prior training, the different values of KL penalty, and different values of $\rho_0$ are 0.89, 0.89, and 0.93 at 1 \% level of significance.  Similar results are found in CNN structure with CIFAR10 dataset where Kendall-tau correlations are 0.89, 0.83, and 0.57.}
 \label{fig:3}
\end{figure*}

 


\begin{table*}[!h]
\centering
\begin{tabular}{|c|c|c|c|c|}
	\hline
	Method & Network \& dataset  & Generalization bound & Testing error & Computation time \\
	\hline\hline
	 &MNIST + FCN &0.0427 &0.0192 &3888 mins\\
	 {Brute-Force}& MNIST + CNN  &0.0247 &0.0091 & 20736 mins\\
    &CIFAR10 + FCN &0.5377 &0.4840 & 8424 mins\\
    &CIFAR10 + CNN &0.1969 &0.1514 & 55080 mins\\
	\hline
	 &MNIST + FCN &0.0478 &0.0213 &1296 mins\\
	 {$\mathcal{PA}$}&MNIST + CNN&0.0346&0.0168 &1296 mins\\
	 &CIFAR10 + FCN &0.5490 &0.4905 &1296 mins\\
    &CIFAR10 + CNN&0.1970&0.1510 &1296 mins\\
	\hline
\end{tabular}
\vspace*{-1mm}
\caption{We conduct the grid search runs over 9 data-dependent prior with different subsets data for prior training, 9 different values of KL penalty, and 8 different values of $\rho_0$. Under any setting, the generalization bound from the hyperparameters selected from $\mathcal{PA}$ is close to the best generalization bound while saving lots of computation time. Training is performed with a server with a CPU with 5,120 cores, and a 32 GB Nvidia Quadro V100. Note, computation time per run using $\mathcal{PA}$ is around 2 mins. Whereas, computation times for the actual posterior training are 6 mins and 32 mins for MNIST trained with FCN and CNN respectively. In addition, 13 mins and 85 mins are used for CIFAR-10 trained with FCN and CNN.}
\label{t1}
\end{table*}

\subsection{Selecting hyperparameters via training-free metric} \label{PFHS}
PAC-Bayesian learning framework provides competitive performance with non-vacuous generalization bounds. However, the tightness of this generalization bound depends on the hyperparameters used, such as the proportionality of data used for the prior, the initialization of $\rho_0$, and the KL penalty weight ($\lambda$). Since these three values do not change during the training, we refer them as hyperparameters. Choosing the right hyperparameters via grid search is obviously prohibitive, as each attempt to compute the generalization bound can involve significant computational resources.

Another plausible approach is to design some kind of predictive, so-called ``training-free'' metric such that we can approximate the error bound without going through an expensive training process. In light of this goal, we have already developed a generalization bound in the theorem \ref{thm:PAC} via NTK. Since NTK changes are held constant during training, we can predict the generalization bound by this proxy metric, which can be formulated as follows:
\begin{equation}
\begin{aligned}
\mathcal{PA}  = &\operatorname{Tr}\bigg(\frac{(\boldsymbol{\Theta}+ {\lambda}/{\sigma^2_0} {\bf I})^{-1} \cdot {\bf Y}   {\bf Y}^\top}{\sigma^2_0 \cdot n} + \\ & \frac{\lambda}{\sigma^2_0}  \sqrt{\frac{ (\boldsymbol{\Theta}+  {\lambda}/{\sigma^2_0} {\bf I})^{-2}\cdot  {\bf Y}  {\bf Y}^\top}{n}} \bigg)
\end{aligned}
\end{equation}

where $\boldsymbol{\Theta}$ denotes the gram matrix. ${\bf Y} {\bf Y}^\top$ is a $n \times n$ label similarity matrix (if two data have the same label, their joint entry in the matrix is 1 and 0 otherwise), and $n$ is the number of data used. 
To demonstrate the computational practicality of this training-free metric, we compute $\mathcal{PA}$ using only a subset of the data for each class (325 per class for FCN and 75 per class for CNN). We should also mention that training-free methods for searching neural architectures are not new, they can be found in NAS \cite{chen2021neural,deshpande2021linearized}, MAE Random Sampling \cite{camero2021bayesian}, pruning at initialization \cite{abdelfattah2021zero}. To the best of our knowledge, there is currently no training-free method for selecting hyperparameters in the PAC-Bayesian framework, which we consider to be one of the novelties of this paper.

The Figure \ref{fig:3} demonstrates a strong correlation between $\mathcal{PA}$ and the actual generalization bound. 
Finally, we demonstrate that by searching through all possible combinations of hyperparameters using $\mathcal{PA}$, it is possible to select a hyperparameter leading towards a result that is comparable to the best generalization bound, but without  the excessive computation. To put things in perspective, in Table \ref{t1}, we show that grid search can be much more expensive than $\mathcal{PA}$. For instance, under the setting of CNN and CIFAR10 dataset, we save 42.5 times computational time to find the bound that is 1\% lower than the best generalization bound.

\section{Discussion}
In this work, we theoretically prove that the learning dynamics of deep probabilistic neural networks using training objectives derived from PAC-Bayes bounds are exactly described by the NTK in over-parameterized setting. Empirical investigation reveals that this agrees well with the actual training process. Furthermore, the expected output function trained with PAC-Bayesian bound converges to the kernel ridge regression under a mild assumption. Based on this finding, we obtain an explicit generalization bound with respect to NTK for PAC-Bayesian learning, which improves over the generalization bound obtained through NTK on non-probabilistic neural network. Finally, we show PAC-Bayesian bound score, the training-free method, can effectively select the hyperparameters which leads to a lower generalization bound without cost excessive computation time that brute-force grid search would take. In a summary, we establish our theoretical analysis on PAC-Bayes with random initialized prior. One promising direction is to study PAC-Bayesian learning with data-dependent prior by NTK.



\bibliography{example_paper}
\bibliographystyle{icml2022}

\appendix
\onecolumn

\section{Preliminary Lemmas} 

In this section, we list preliminary lemmas that with be used in the proof.

\begin{lemma} [Markov's inequality]
If $X$ is a nonnegative random variable and $a > 0$, then the probability that $X$ is at least $a$ is at most the expectation of $X$ divided by $a$:
$$
P(X \ge a) \le \frac{\mathbb{E}[X]}{a}
$$
\end{lemma}

\begin{lemma} [Hoeffding's inequality]
Let $X_1, \ldots, X_n$ be independent random variables such that $ a_{i}\leq X_{i} \leq b_{i}$ almost surely. Consider the sum of these random variables, $S_n = X_1 + \cdots + X_n$, then Hoeffding's theorem states that, for all $t > 0$:
$$
P( |S_n -\mathbb{E}[S_n]| \ge t) \le 2 \exp \bigg(-\frac{2t^2}{\sum_{i=1}^n (b_i - a_i)^2} \bigg)
$$
\end{lemma}

\section{Mission Proofs of Section \ref{sec:4.1}}
This section contains detailed proofs of the results that are missing in the main paper.

\begin{lemma} [PNTK at initialization] \label{lemma:initial}
If $m = \Omega(\frac{n^2}{\lambda^2_0}\log\frac{n}{\delta})$, while $\boldsymbol{\mu}_r$ and $\boldsymbol{\sigma}_r$ are initialized by the form in Theorem \ref{thm:opt}, then with probability at least $1-\delta$, $\big \|\boldsymbol{\Theta}^\mu({\bf X}, {\bf X},0) -\boldsymbol{\Theta}^\infty({\bf X}, {\bf X}) \big \|_2  \le \frac{\lambda_0}{4}$, $\big \|\boldsymbol{\Theta}^\mu({\bf X}, {\bf X},0) \big \|_2  \ge \frac{3\lambda_0}{4}$;
and $\big \|\boldsymbol{\Theta}^\sigma_0({\bf X}, {\bf X}) -\boldsymbol{\Theta}^\infty({\bf X}, {\bf X}) \big \|_2   \le \frac{\lambda_0}{4}$, $\big \|\boldsymbol{\Theta}^\sigma({\bf X}, {\bf X},0) \big \|_2  \ge \frac{3\lambda_0}{4}$.
\end{lemma}

\begin{proof}
The derivative of output over the parameters $\boldsymbol{\mu}$ and $\boldsymbol{\sigma}$ can be expressed as
\begin{equation} \label{eq:deri}
\frac{\partial f({\bf x}_i)}{\partial \boldsymbol{\mu}_r} = \frac{1}{\sqrt{m}} { v}_r  \mathbb{I} ({\bf w}^\top_r {\bf x}_i \ge 0) {\bf x}_i ~~~ \frac{\partial f({\bf x}_i)}{\partial \boldsymbol{\sigma}_r} = \frac{1}{\sqrt{m}} { v}_r  \mathbb{I} ({\bf w}^\top_r {\bf x}_i) {\bf x}_i \odot \boldsymbol{\epsilon}_r 
\end{equation}

Since the probabilistic neural network contains two set of parameters, we discuss the corresponding NTK accordingly.

\paragraph{(1) $\boldsymbol{\Theta}^\mu$.} Plugging the derivative result regarding mean weights in Equation (\ref{eq:deri}) into the definition of PNTK (Eqution \ref{eq:NTK_kernel}) yields:
$$
\boldsymbol{\Theta}^\mu_{ij}(0) = \frac{{\bf x}^\top_i {\bf x}_j}{m} ({v}_r)^2  \sum_{r=1}^m \mathbb{I}( {\bf w}^\top_r {\bf x}_i \ge 0) \mathbb{I}( {\bf w}^\top_r {\bf x}_j \ge 0) = \frac{{\bf x}^\top_i {\bf x}_j}{m}  \sum_{r=1}^m \mathbb{I}( {\bf w}^\top_r {\bf x}_i \ge 0) \mathbb{I}( {\bf w}^\top_r {\bf x}_j \ge 0) $$ 
By an analysis, we find that for all pairs of $i,j$, $\boldsymbol{\Theta}^\mu_{ij}(0)$ is the average of $m$ i.i.d. random variables bounded in $[0,1]$, with the expectation $\boldsymbol{\Theta}^\infty_{ij} = \mathbb{E}_{{\bf w}  } \big[{\bf x}^\top_i {\bf x}_j \mathbb{I} \{ {\bf w}^\top {\bf x}_i \ge 0, {\bf w}^T {\bf x}_j \ge 0 \} \big]$. Here ${\bf w}_i = \boldsymbol{\mu}_i + \boldsymbol{\sigma}_i  {\epsilon}$, 
and $\boldsymbol{\mu}_i \sim \mathcal{N}(\hat \mu, \hat \sigma) $. Then with a simple argument, we show that ${\bf w}$ is also a Gaussian variable.
The variance of distribution for ${\bf w}_i$ can be inferred through the following process,
$$
f({\bf w}_i) = \frac{1}{\sqrt{2\pi} \sigma_0} e^{-\frac{({\bf w}_i-\boldsymbol{\mu}_i)^2}{2 \sigma_0^2}} ~~~f(\boldsymbol{\mu}_i) = \frac{1}{\sqrt{2\pi} \hat \sigma} e^{-\frac{(\boldsymbol{\mu}_i- \hat \mu)^2}{2 \hat \sigma^2}}
$$
Taking the density function of variable $\boldsymbol{\mu}_i$ into $f({\bf w}_i)$, we can obtain,
$$
\begin{aligned}
f({\bf w}_i) = \int \frac{1}{\sqrt{2\pi} \sigma_0} e^{-\frac{({\bf w}_i-\boldsymbol{\mu}_i)^2}{2 \sigma_0^2}} \frac{1}{\sqrt{2\pi} \hat \sigma} e^{-\frac{(\boldsymbol{\mu}_i- \hat \mu)^2}{2 \hat \sigma^2}} d \boldsymbol{\mu}_i = \frac{1}{\sqrt{2\pi (\sigma_0^2 + \hat \sigma^2)}} e^{-\frac{({\bf w}_i-\hat \mu)^2}{2 (\sigma_0^2+ \hat \sigma^2)}} 
\end{aligned}
$$
Then by Hoeffding's inequality,  we know that the following inequality holds with probability at least $1-\delta'$,
$$
\big |\boldsymbol{\Theta}^\mu_{ij}(0) - \boldsymbol{\Theta}^\infty_{ij} \big| \le \sqrt{\frac{\log(2/\delta')}{2m}}
$$

Because NTK matrix is of size $n \times n$, we then apply a union bound over all $i,j \in [n]$  (by setting $\delta'=\delta/n^2$), and obtain that 


$$
\big|\boldsymbol{\Theta}^\mu_{ij}(0) - \boldsymbol{\Theta}^\infty_{ij} \big| \le \sqrt{\frac{\log(2n^2/\delta)}{2m}}
$$
Thus we have,
$$
\big \| \boldsymbol{\Theta}^\mu(0)- \boldsymbol{\Theta}^\infty \big \|^2_2 \le  \big \| \boldsymbol{\Theta}^\mu(0)- \boldsymbol{\Theta}^\infty \big \|^2_F \le \sum_{i,j} \big| \boldsymbol{\Theta}^\mu_{ij}(0) - \boldsymbol{\Theta}_{ij}^\infty \big |^2 = O\bigg(\frac{n^2 \log(2n^2/\delta)}{m}\bigg)
$$
Finally, if $\sqrt{\frac{n^2 \log(2n^2/\delta)}{m}} \le \frac{\lambda_0}{4}$, which implies  $m =O \big(\frac{n^2\log(n^2/\delta)}{\lambda^2_0} \big)$, then with probability at least $1-\delta$,
$$
\big \|\boldsymbol{\Theta}^\mu({\bf X}, {\bf X},0) -\boldsymbol{\Theta}^\infty({\bf X}, {\bf X}) \big \|_2  \le \frac{\lambda_0}{4},~ \big \|\boldsymbol{\Theta}^\mu({\bf X}, {\bf X},0) \big \|_2  \ge \frac{3\lambda_0}{4}.
$$

\paragraph{(2) $\boldsymbol{\Theta}^\sigma$.} Plugging the derivative result regarding mean weights in Equation (\ref{eq:deri}) into the definition of PNTK (Eqution \ref{eq:NTK_kernel}) yields:
$$
\boldsymbol{\Theta}^\sigma_{ij} = \frac{{\bf x}^\top_i {\bf x}_j}{m} \sum_{r=1}^m \mathbb{I}( {\bf w}^\top_r {\bf x}_i \ge 0) \mathbb{I}( {\bf w}^\top_r {\bf x}_j \ge 0) \boldsymbol{\epsilon}^2_r
$$
Note that the tangent kernel for $\sigma$ differs from $\boldsymbol{\Theta}^\mu_{ij}$ with an additional term $\boldsymbol{\epsilon}^2_r$. It is known the $\boldsymbol{\epsilon}^2_r \sim \chi_1$ independently with $\mathbb{I}( {\bf w}^\top_r {\bf x}_i \ge 0) \mathbb{I}( {\bf w}^\top_r {\bf x}_j \ge 0)$. Because $\mathbb{E}[\chi_1] = 1$, the expectation of $\boldsymbol{\Theta}^\sigma_{ij}$ equals the expectation of $\boldsymbol{\Theta}^\mu_{ij}$. Thus for all pairs of $i,j$, $\boldsymbol{\Theta}^\sigma_{ij}$ is the average of $m$ i.i.d. random variables with the expectation $\boldsymbol{\Theta}^\infty_{ij} = \mathbb{E}_{{\bf w}  } \big[{\bf x}^\top_i {\bf x}_j \mathbb{I} \{ {\bf w}^\top {\bf x}_i \ge 0, {\bf w}^\top {\bf x}_j \ge 0 \} \big]$. 

Now we calculate the concentration bound. It is known that $\boldsymbol{\epsilon}^2_r$ is independent and sub-exponential.
Then, by sub-exponential tail bound, we know that the following holds with probability at least $1-\delta'$,
$$
\big| \boldsymbol{\Theta}^\sigma_{ij}(0) - \boldsymbol{\Theta}^\infty_{ij} \big| \le \sqrt{\frac{\log(8/\delta')}{2m}}
$$
This bound is of the same order to concentration bound for $ \boldsymbol{\Theta}^\mu_{ij}(0)$. Thus we can take all the arguments for $ \boldsymbol{\Theta}^\mu_{ij}(0)$ above to finalize the proof: If $\sqrt{\frac{n^2 \log(8n^2/\delta)}{m}} \le \frac{\lambda_0}{4}$, which implies  $m =O \big(\frac{n^2\log(n^2/\delta)}{\lambda^2_0} \big)$, then with probability at least $1-\delta$,
$$
\big \|\boldsymbol{\Theta}^\sigma({\bf X}, {\bf X},0) -\boldsymbol{\Theta}^\infty({\bf X}, {\bf X}) \big \|_2  \le \frac{\lambda_0}{4},~ \big \|\boldsymbol{\Theta}^\sigma({\bf X}, {\bf X},0) \big \|_2  \ge \frac{3\lambda_0}{4}.
$$
\end{proof}

\begin{lemma} [PNTK during training] 
If $\boldsymbol{\mu}_r$ and $\boldsymbol{\sigma}_r$ are initialized the same with Theorem \ref{thm:opt}, for any set of weight vectors ${\bf w}_1,\ldots, {\bf w}_m $ that satisfy for any $r \in [m]$, $\| {\bf w}_{r}(t) - {\bf w}_{r}(0) \| \le \frac{c \lambda_0 \delta \sqrt{\sigma_0^2+ \hat \sigma^2}}{n^2} \equiv R$, where $c$ is a constant, then with probability at least $1-\delta$, $\big \|\boldsymbol{\Theta}^\mu({\bf X}, {\bf X},t) -\boldsymbol{\Theta}^\infty({\bf X}, {\bf X}) \big \|_2  \le \frac{\lambda_0}{4}$, $\big \|\boldsymbol{\Theta}^\mu({\bf X}, {\bf X},t) \big \|_2  \ge \frac{\lambda_0}{2}$;
and $\big \|\boldsymbol{\Theta}^\sigma({\bf X}, {\bf X},t) -\boldsymbol{\Theta}^\infty({\bf X}, {\bf X}) \big \|_2   \le \frac{\lambda_0}{4}$, $\big \|\boldsymbol{\Theta}^\sigma({\bf X}, {\bf X},t) \big \|_2  \ge \frac{\lambda_0}{2}$.
\end{lemma}

\begin{proof}
Define the event $A_{ir} = \{ \exists {\bf w}: \|{\bf w}_{r}(t) -{\bf w}_{r}(0) \| \le R, \mathbb{I}\{{\bf x}^\top_i {\bf w}_{r}(t) \ge 0 \} \neq \mathbb{I}\{{\bf x}^\top_i {\bf w}_{r}(0) \ge 0 \}  \}$. Then this event would only happen when $|{\bf w}_{r,0}^\top {\bf x}_i| < R $. As calculated in the proof of Lemma \ref{lemma:initial}, we know that ${\bf w}_i \sim \mathcal{N}(0, \sqrt{\hat \sigma^2 + \sigma^2_0 })$. Then with the condition that $\|{\bf x}_i \| = 1$, we obtain 
$$
z = {\bf w}_{r}(0)^\top {\bf x}_i \sim \mathcal{N} \big(0, \sqrt{\hat \sigma^2 + \sigma^2_0 } \big)
$$
By integrating over the condition $|{\bf w}_{r}(0)^\top {\bf x}_i| < R $, we have  
$$
P(A_{ir}) = P_{z \sim \mathcal{N}(0, \sqrt{\sigma^2_0 + \hat \sigma^2})} (|z| < R) \le \frac{2R}{\sqrt{2 \pi (\sigma^2_0+ \hat \sigma^2)}} 
$$ 

Then we can bound the entry of the matrix $\boldsymbol{\Theta}^\mu$:
$$
 \begin{aligned}
&\mathbb{E} \big[|\boldsymbol{\Theta}^\mu_{ij}(0) -\boldsymbol{\Theta}^\mu_{ij}(t) | \big] = \mathbb{E} \bigg[\frac{1}{m}\big |{\bf x}^\top_i {\bf x}_j \sum_{r=1}^m \big( \mathbb{I} \{{\bf w}^\top_{r}(0) {\bf x}_i \ge 0, {\bf w}^\top_{r}(0) {\bf x}_j \ge 0 \} - \mathbb{I} \{ {\bf w}_{r}(t)^\top {\bf x}_i \ge 0, {\bf w}_{r}(t)^\top {\bf x}_j \ge 0 \} \big) \big| \bigg] \\
& \le \frac{1}{m} \sum_{r=1}^m \mathbb{E} \big[\mathbb{I} \{A_{ir} \cup A_{jr} \} \big] \le \frac{4R}{\sqrt{2 \pi (\sigma^2_0+ \hat \sigma^2)}}
\end{aligned}
$$
Summing over all entries of the matrix, we have 
$\mathbb{E} \bigg[\sum_{(ij)} \big|\boldsymbol{\Theta}^\mu_{ij}(0) -\boldsymbol{\Theta}^\mu_{ij}(t) \big| \bigg] \le \frac{4n^2R}{\sqrt{2 \pi (\sigma_0^2+ \hat \sigma^2)}}$. By Markov's inequality, with probability of $1-\delta$ over distribution of ${\bf w}_{r,0}$ for $r \in [m]$, $\sum_{(ij)} \big| \boldsymbol{\Theta}^\mu_{ij}(0) -\boldsymbol{\Theta}^\mu_{ij}(t) \big| \le \frac{4n^2R}{\sqrt{2 \pi   (\sigma^2_0+ \hat \sigma^2)}\delta}$. Then by the matrix perturbation theory, we have, 
$$
\big \| \boldsymbol{\Theta}^\mu(t) - \boldsymbol{\Theta}^\mu(0) \big \|_2 \le \big \| \boldsymbol{\Theta}^\mu(t) - \boldsymbol{\Theta}^\mu(0) \big \|_F \le \sum_{(ij)} \big| \boldsymbol{\Theta}^\mu_{ij}(t) -\boldsymbol{\Theta}^\mu_{ij}(0)
\big| \le  \frac{4n^2R}{\sqrt{2 \pi (\sigma_0^2+ \hat \sigma^2)}\delta}
$$

If $\frac{4n^2R}{\sqrt{2 \pi (\sigma_0^2+ \hat \sigma^2)}\delta} \le \frac{\lambda_0}{4}$, which implies $\| {\bf w}_{r}(t) - {\bf w}_{r}(0) \|  \le \frac{c \lambda_0 \delta \sqrt{\sigma_0^2+ \hat \sigma^2}}{n^2}$, then with probability at least $1-\delta$,
$$
\big \|\boldsymbol{\Theta}^\mu({\bf X}, {\bf X},t) -\boldsymbol{\Theta}^\infty({\bf X}, {\bf X}) \big \|_2  \le \frac{\lambda_0}{4},~ \big \|\boldsymbol{\Theta}^\mu({\bf X}, {\bf X},t) \big \|_2  \ge \frac{\lambda_0}{2}.
$$

On the other hand, we can bound the entry of the matrix $\boldsymbol{\Theta}^\sigma$ with following inequality:
$$
 \begin{aligned}
&\mathbb{E} \big[|\boldsymbol{\Theta}^\sigma_{ij}(0) -\boldsymbol{\Theta}^\sigma_{ij}(t) | \big] = \mathbb{E} \bigg[\frac{1}{m}\big |{\bf x}^\top_i {\bf x}_j \sum_{r=1}^m \big( \mathbb{I} \{{\bf w}^\top_{r}(0) {\bf x}_i \ge 0, {\bf w}^\top_{r}(0) {\bf x}_j \ge 0 \} \boldsymbol{\epsilon}^2_{r}(0) - \mathbb{I} \{ {\bf w}_{r}(t)^\top {\bf x}_i \ge 0, {\bf w}_{r}(t)^\top {\bf x}_j \ge 0 \} \boldsymbol{\epsilon}^2_{r}(t) \big)  \big| \bigg] \\
& \le \frac{1}{m} \sum_{r=1}^m \mathbb{E} \big[\mathbb{I} \{A_{ir} \cup A_{jr} \} \boldsymbol{\epsilon}^2_{r}(t) \big] \le \frac{4R}{\sqrt{2 \pi (\sigma^2_0+ \hat \sigma^2)}}
\end{aligned}
$$

For the first inequality in the second lien, we use $\boldsymbol{\epsilon}^2_{r}(t) = \boldsymbol{\epsilon}^2_{r}(0)$ according to the definition of ${\bf w}_{r}(t)$. We argue that our definition is reasonable because $\mathbb{E}[\boldsymbol{\epsilon}^2_{r}(0)] = \mathbb{E}[ \boldsymbol{\epsilon}^2_{r}(t)]$ when $\boldsymbol{\epsilon}_{r}(t)$ and $\boldsymbol{\epsilon}_{r}(0)$ are two i.i.d. variables. With our definition, we can compute $\| {\bf w}_{r}(t) - {\bf w}_{r}(0) \|$ directly, and keep the corresponding conclusion applicable to real weights at training step simultaneously.

The above inequality shows that with additional $\boldsymbol{\epsilon}^2_r$, we still have the same result. Once again, we can apply arguments for $\boldsymbol{\Theta}^\mu$ here, and conclude the proof: If $\frac{4n^2R}{\sqrt{2 \pi (\sigma_0^2+ \hat \sigma^2)}\delta} \le \frac{\lambda_0}{4}$, which implies $\| {\bf w}_{r}(t) - {\bf w}_{r}(0) \|  \le \frac{c \lambda_0 \delta \sqrt{\sigma_0^2+ \hat \sigma^2}}{n^2}$, then with probability at least $1-\delta$,
$$
\big \|\boldsymbol{\Theta}^\sigma({\bf X}, {\bf X},t) -\boldsymbol{\Theta}^\infty({\bf X}, {\bf X}) \big \|_2  \le \frac{\lambda_0}{4},~ \big \|\boldsymbol{\Theta}^\sigma({\bf X}, {\bf X},t) \big \|_2  \ge \frac{\lambda_0}{2}.
$$
\end{proof}

\begin{lemma} [Change of weights during training]
Suppose $\lambda_{\min}(\boldsymbol{\Theta}(t)) \ge \frac{\lambda_0}{2}$ for $0 < t < T $. Then we have $  \widehat{R}_S(Q,t) \le \exp(-{\lambda_0} t)\widehat{R}_S(Q,0) $, and with probability at least $1-\delta$ over initialization, $\| {\bf w}_{r}(t) - {\bf w}_{r}(0) \|_2 \le \frac{2n}{\sqrt{m \delta}\lambda_0 }  \equiv R' $.
\end{lemma}

\begin{proof}
According to the gradient flow of output function, we have
$$
\begin{aligned}
 \frac{d  f({\bf x}_i,t)}{d t}  & = \sum_{r=1}^m \big( \langle \frac{ \partial   f({\bf x}_i,t)}{ \partial \boldsymbol{\mu}_r} , \frac{d \boldsymbol{\mu}_r(t)}{d t} \rangle + \langle \frac{ \partial   f({\bf x},t)}{ \partial\boldsymbol{\sigma}_r} , \frac{d \boldsymbol{\sigma}_r(t)}{d t} \rangle \big)  \\ & = \sum_{j=1}^n ({\bf y}_i-f({\bf x}_j)) \sum_{r=1}^m \big( \langle \frac{ \partial  f({\bf x}_i)}{ \partial \boldsymbol{\mu}_r} , \frac{  f({\bf x}_j)}{ \partial \boldsymbol{\mu}_r} \rangle  + \langle \frac{ \partial f({\bf x}_i)}{ \partial \boldsymbol{\sigma}_r} ,  \frac{ \partial  f({\bf x}_j)}{ \partial \boldsymbol{\sigma}_r} \rangle \big) \\
 & = \sum_{j=1}^n ({\bf y}_j-  f({\bf x}_j,t)) (\boldsymbol{\Theta}^{\mu}_{ij} + \boldsymbol{\Theta}^{\sigma}_{ij})
\end{aligned}
$$

Then the dynamics of loss can be calculated,
$$
\begin{aligned}
\frac{d}{dt}  \widehat{R}_S(Q,t) &= \frac{1}{2}\frac{d}{dt}  \big \| \mathbb{E}_{f(t) \sim Q} f({\bf X},t) - {\bf Y} \big \|_2^2 = \frac{1}{2}\frac{d}{dt}  \big \| \widehat f({\bf X},t) - {\bf Y} \big \|_2^2 \\ &= - ({\bf Y}- \widehat f({\bf X},t))^\top (\boldsymbol{\Theta}^\mu(t) + \boldsymbol{\Theta}^\sigma(t))  ({\bf Y}- \widehat f({\bf X}),t) \le -\lambda_0 \| {\bf Y} -\widehat f ({\bf X},t)\|^2_2
\end{aligned}
$$
where we define $\widehat f(t) = \mathbb{E}_{f(t) \sim Q} f(t) = \mathbb{E}_\epsilon f(t) $. ${\bf X}, {\bf Y}$ are collections of features and labels. Then the loss can be bounded,
$$
\widehat{R}_S(Q,t) = \frac{1}{2}\| {\bf Y}-\widehat f({\bf X},t) \|^2_2 \le \exp(-{\lambda_0} t)\widehat{R}_S(Q,0)
$$
The above equation implies the linear convergence rate of gradient descent on over-parameterized probabilistic network.

Now we bounded the change of mean weights:
$$
\begin{aligned}
\bigg \| \frac{d}{dt} \boldsymbol{\mu}_{r}(t) \bigg \|_2 & = \bigg \| \sum_{i=1}^n (y_i -\widehat f({\bf x}_i )) \frac{1}{\sqrt{m}} {\bf v}_r {\bf x}_i \mathbb{I}({\bf w}^\top_{r}(t) {\bf x}_i \ge 0) \bigg \|_2 \\
& \le \frac{1}{\sqrt{m}} \sum_{i=1}^n |y_i -\widehat f({\bf x}_i)| \le \sqrt{\frac{n}{m}} \| {\bf Y}-\widehat f({\bf X},t) \|_2 \le \sqrt{\frac{n}{m}} \exp(-\lambda_0 t/2) \| {\bf Y}-\widehat f({\bf X},0) \|_2
\end{aligned}
$$
Integrating the gradient, we have $$
\|\boldsymbol{\mu}_{r}(T) -\boldsymbol{\mu}_{r}(0) \|_2 \le \int_{0}^T  \big \|  \frac{d}{dt} \boldsymbol{\mu}_{r}(t) \big \|_2 dt  \le  2 \sqrt{\frac{n}{m}} \frac{\| {\bf y}-\widehat f({\bf X},0) \|_2}{\lambda_0} 
$$

Next we bounded the change of variance weights:
$$
\begin{aligned}
\bigg \| \frac{d}{dt} \boldsymbol{\sigma}_r(t) \bigg \|_2 & = \bigg \| \sum_{i=1}^n (y_i -\widehat f({\bf x}_i )) \frac{1}{\sqrt{m}} \mathbb{E}_{\boldsymbol{\epsilon}_{r}(t) } {\bf v}_r {\bf x}_i \odot \boldsymbol{\epsilon}_{r}(t)  \mathbb{I}( {\bf w}^\top_{r}(t) {\bf x}_i \ge 0)  \bigg \|_2  = 0
\end{aligned}
$$
In the above derivation, we use the definition of loss which is $\widehat{R}_S(Q) = \mathbb{E}_{f \in Q} \widehat{R}_S(f)$ and interchange integration and differentiation. The result leads to:
$$
\|\boldsymbol{\sigma}_{r}(T) -\boldsymbol{\sigma}_{r}(0) \| \le \int_{0}^T  \bigg \|  \frac{d}{dt} \boldsymbol{\sigma}_{r}(t) \bigg\|_2 dt  = 0
$$

Plug the results for both mean and variance weights together, we obtain
$$
\|{\bf w}_{r}(T) -{\bf w}_{r}(0) \| = \|{\boldsymbol \mu}_{r}(T) -{\boldsymbol \mu}_{r}(0) + {\boldsymbol \epsilon}_{r}(0) \odot ({\boldsymbol \sigma}_{r}(T)- {\boldsymbol \sigma}_{r}(0)) \| \le
2 \sqrt{\frac{n}{m}} \frac{\| {\bf Y}-\widehat f({\bf X},0) \|_2}{\lambda_0} =  \frac{4\sqrt{n \widehat{R}_S(Q,0)}}{\sqrt{m}\lambda_0} 
$$
Finally, it is time to bound $\widehat{R}_S(Q,0)$,
$$
\mathbb{E} \big[\| {\bf Y} - \widehat f({\bf X}) \|^2_2 \big] = \sum_{i=1}^n (y^2_i + y_i \mathbb{E}[\widehat{f}( {\bf x}_i )] + \mathbb{E}[\widehat{f}( {\bf x}_i )^2]) = \sum_{i=1}^n (1+O(1) = O(n))
$$
Thus by Markov's inequality, we have with probability at least $1-\delta$, $\widehat{R}_S(Q,0) = O(\frac{n}{\delta})$, and 
$$
\| {\bf w}_{r}(t) - {\bf w}_{r}(0) \|_2 \le \frac{2n}{\sqrt{m \delta}\lambda_0 }  \equiv R'
$$
\end{proof}

\begin{lemma}
If $R' < R$, we have
$\widehat{R}_S(Q,t)  \le \exp(-\lambda_0 t) \widehat{R}_S(Q,0). 
$
\end{lemma}

\begin{proof}
Since $R' = \frac{2n}{\sqrt{m \delta}\lambda_0 } $ and $R = \frac{c \lambda_0 \delta \sqrt{(\sigma^2_0+ \hat \sigma^2)}}{n^2} $, we have $m = \Omega\big(\frac{n^6  }{ \lambda^4_0 \delta^4 (\sigma_0^2+\hat \sigma^2)}\big)$.
 \end{proof}
 
 \section{Missing Proofs of Section \ref{sec:4.2}}

\begin{theorem} 
Consider gradient descent on target function (\ref{eq:target}), with the initialization stated in Theorem \ref{thm:opt}. Suppose $m \ge {\rm poly}({n},1/\lambda_0, 1/\delta,1/ \mathcal{E})$. Then with probability at least $1-\delta$ over the random initialization, we have
\begin{equation}
 {\widehat{f}({\bf x})} = \boldsymbol{\Theta}^\infty({\bf x},{\bf X}) \big (\boldsymbol{\Theta}({\bf X},{\bf X}) + \lambda/{\sigma^2_0} {\bf I} \big)^{-1} {\bf Y} \pm \mathcal{E}
\end{equation}
where $\widehat{f}({\bf x}) = \mathbb{E}_{f \sim Q} f(\bf x) $ aligns with the definition of empirical loss function. 
\end{theorem}

\begin{proof} 
At a high level, our proof first establishes the result of kernel ridge regression in the infinite-width limit, then bounds the perturbation on the predict.

According the linearization rules for infinitely-wide networks \citep{lee2019wide}, the output function can be expressed as,
$$
\widehat{f}_{\rm ntk}({\bf x},t) = \phi_{{\mu}}({\bf x})^\top(\boldsymbol{\mu}(t)-\boldsymbol{\mu}(0))+\phi_{{\sigma}}({\bf x})^\top(\boldsymbol{\sigma}(t)-\boldsymbol{\sigma}(0)),
$$
where $\phi_{{\mu}}({\bf x}) = \nabla_{{{\mu}}} \widehat f({\bf x},0)$, and $\phi_{{\sigma}}({\bf x}) = \nabla_{{{\sigma}}} \widehat f({\bf x},0)$.
Since $\boldsymbol{\sigma}$ does not change during training, then the KL divergence reduces to
$$
\text{KL} = \frac{1}{2}   \frac{({\mu}_{t}-{\mu}_{0})^2}{{\sigma}^2_{0}} 
$$
Then the gradient flow equation for $\boldsymbol{\mu}$ becomes,

$$
\begin{aligned}
\frac{d \boldsymbol{\mu}(t)}{d t} = \frac{\partial \mathcal{B}}{ \partial   \boldsymbol{\mu}} & = (\widehat f_{\rm ntk}({\bf X},t) -{\bf Y})\phi_{{\mu}}({\bf X}) + {\lambda}/{\sigma^2_0}(\boldsymbol{\mu}(t)-\boldsymbol{\mu}(0))  \\ &= \boldsymbol{\Theta}^\mu(\boldsymbol{\mu}(t)-\boldsymbol{\mu}(0)) -\phi_{\mu}({\bf X})^\top {\bf Y} + {\lambda}/{\sigma^2_0}(\boldsymbol{\mu}(t)-\boldsymbol{\mu}(0))
\end{aligned}
$$
By analysing above equation, we find it is an ordinary differential equation regarding $\boldsymbol{\mu}_t$, and the solution is,
$$
\boldsymbol{\mu}(t) = \phi_{\mu}({\bf X})^\top(\boldsymbol{\Theta}({\bf X},{\bf X})+\lambda/\sigma^2_0 {\bf I})^{-1}  {\bf Y}({\bf I}-e^{-(\boldsymbol{\Theta}({\bf X},{\bf X})+\lambda/\sigma^2_0 {\bf I})t})
$$
Plug this result into the linearization of expected output function, we have,
$$
\widehat f_{\rm ntk}({\bf x},t) = \boldsymbol{\Theta}({\bf x},{\bf X})(\boldsymbol{\Theta}({\bf X},{\bf X})+{\lambda}/{\sigma^2_0} {\bf I})^{-1} {\bf Y} ({\bf I}-e^{-(\boldsymbol{\Theta}({\bf X},{\bf X})+{\lambda}/{\sigma^2_0} {\bf I})t})
$$
when we take the time to be infinity, 
$$
\widehat f_{\rm ntk}({\bf x},t =\infty) = \boldsymbol{\Theta}({\bf x},{\bf X})(\boldsymbol{\Theta}({\bf X},{\bf X})+{\lambda}/{\sigma^2_0} {\bf I})^{-1} {\bf y} .
$$
The next step is to show that
$$
|{\widehat f_{\rm ntk}({\bf x}) - \widehat f({\bf x})}| \le
O(\mathcal{E}).
$$
where $\mathcal{E} = \mathcal{E}_{\rm init}  + \frac{\sqrt{n}}{\lambda_0^2}\log(\frac{n}{\mathcal{E}_{\Theta} \lambda_0})\mathcal{E}_{\Theta}$. In which, we define $|\widehat f({\bf x},0)| \le \mathcal{E}_{\rm init}$ and $\|\boldsymbol{\Theta}^\infty -\boldsymbol{\Theta}(t)  \|_2 \le \mathcal{E}_{\Theta}$. 

The proof relies a careful analysis on the trajectories induced by gradient flows for optimizing the neural network and the NTK predictor. The detailed proof can be found in the proof of Theorem 3.2 in \cite{arora2019exact}, and we can replace kernel ridge regression here by kernel regression.

\end{proof}

 \section{Missing Proofs of Section \ref{sec:4.3}}

\begin{theorem} [PAC-Bayesian bound with NTK] \label{thm:PAC}
Suppose data $S = \{ ({\bf x}_i,y_i)\}_{i=1}^n$ are i.i.d. samples from a non-degenerate distribution $\mathcal{D}$, and $m \ge {\rm poly}(n, \lambda_0^{-1}, \delta^{-1})$. Consider any loss function $\ell: \mathbb{R} \times \mathbb{R} \rightarrow [0,1]$. Then with probability at least $1-\delta$ over the random initialization and the training samples, the probabilistic network trained by gradient descent for $T \ge \Omega(\frac{1}{\eta \lambda_0} \log \frac{n}{\delta})$ iterations has population risk $R_{\mathcal{D}}(Q)$ that is bounded as follows:

\begin{equation}
\begin{aligned}
   R_{\mathcal{D}} & \le  \frac{1}{\sigma^2_0}\frac{  {\bf Y}^\top (\boldsymbol{\Theta}+ {\lambda}/{\sigma^2_0} {\bf I})^{-1} {\bf Y}}{n} + \frac{ \lambda}{\sigma^2_0}  \sqrt{ \frac{ {\bf Y}^\top(\boldsymbol{\Theta}+  {\lambda}/{\sigma^2_0} {\bf I})^{-2} {\bf Y}} {n}}  + O \bigg(\frac{\log \frac{2\sqrt{n}}{\delta}}{n} \bigg).
   \end{aligned}
\end{equation}
\end{theorem}

\begin{proof}
The generalization bound consists two terms, one is the empirical error, and another is KL divergence.

(1) We first bound the empirical error $\sqrt{\sum_{i=1}^n (\widehat f_{\rm ntk}({\bf x}_i,\infty)-y_i)^2}$ with following inequality,
$$
\begin{aligned}
\sqrt{\sum_{i=1}^n (\widehat f_{\rm ntk}({\bf x}_i,\infty)-y_i)^2} & = \| \boldsymbol{\Theta}({\bf X},{\bf X})(\boldsymbol{\Theta}({\bf X},{\bf X}) + {\lambda}/{\sigma^2_0} {\bf I})^{-1} {\bf Y} -{\bf Y} \| 
= \| {\lambda}/{\sigma^2_0}(\boldsymbol{\Theta}+ {\lambda}/{\sigma^2_0} {\bf I})^{-1} {\bf Y} \| \\
& = {\lambda}/{\sigma^2_0}  \sqrt{ {\bf Y}^\top(\boldsymbol{\Theta}+  {\lambda}/{\sigma^2_0} {\bf I})^{-2}{\bf Y}}  
\end{aligned}
$$

Then we can further bound the error as,
$$
\frac{1}{n} \sum_{i=1}^n \ell (\widehat f_{\rm ntk}({\bf x}_i) -y_i) \le \frac{1}{n}\sum_{i=1}^n |\widehat f_{\rm ntk}({\bf x}_i) - y_i| \le \frac{1}{\sqrt{n}} \sqrt{\sum_{i=1}^n |\widehat f_{\rm ntk}({\bf x}_i) -y_i |^2} \le \frac{\lambda}{\sigma^2_0}  \sqrt{\frac{ {\bf Y}^\top(\boldsymbol{\Theta}+  {\lambda}/{\sigma^2_0} {\bf I})^{-2} {\bf Y}}{n}}
$$

(2) The next step is to calculate the KL divergence. According to the solution of differential equation, we have,
$$
\boldsymbol{\mu}(t) -  \boldsymbol{\mu}(0)= \phi_{{\mu}}({\bf x})^\top {\bf Y} (\boldsymbol{\Theta}({\bf X},{\bf X})+ {\lambda}/{\sigma^2_0} {\bf I})^{-1}( {\bf I}-e^{-(\boldsymbol{\Theta}({\bf X},{\bf X})+ {\lambda}/{\sigma^2_0} {\bf I})t}),
$$
then $t= T= \infty$ yields $\boldsymbol{\mu}(t) -  \boldsymbol{\mu}(0)=  \phi_{{\mu}}^\top(\boldsymbol{\Theta}({\bf X},{\bf X})+\frac{\lambda}{\sigma^2_0} {\bf I})^{-1} {\bf Y}$. Therefore, the KL divergence is,
$$
\begin{aligned}
\text{KL} & = \frac{1}{ \sigma^2_0} {\bf Y}^\top (\boldsymbol{\Theta}({\bf X},{\bf X})+{\lambda}/{\sigma^2_0} {\bf I})^{-1} \boldsymbol{\Theta}({\bf X},{\bf X}) (\boldsymbol{\Theta}({\bf X},{\bf X})+{\lambda}/{\sigma^2_0} {\bf I})^{-1} {\bf Y} \le \frac{1}{ \sigma^2_0} {\bf Y}^\top (\boldsymbol{\Theta}({\bf X},{\bf X})+{\lambda}/{\sigma^2_0} {\bf I})^{-1} {\bf Y}
\end{aligned}
$$

Finally, we achieve the PAC-Bayesian generalization bound,
$$
\begin{aligned}
   R_{\mathcal{D}} & \le \frac{{\bf Y}^\top (\boldsymbol{\Theta}+ \frac{\lambda}{\sigma^2_0} {\bf I})^{-1} {\bf Y}}{n\sigma^2_0} + \frac{ \lambda}{\sigma^2_0}  \sqrt{ \frac{ {\bf Y}^\top(\boldsymbol{\Theta}+  \frac{\lambda}{\sigma^2_0} {\bf I})^{-2} {\bf Y}} {n}} +O \bigg(\frac{\log \frac{2\sqrt{n}}{\delta}}{n} \bigg).
   \end{aligned}
$$

\end{proof}

\begin{theorem} [Rademacher bound with NTK] 
Suppose data $S = \{ ({\bf x}_i,y_i)\}_{i=1}^n$ are i.i.d. samples from a non-degenerate distribution $\mathcal{D}$, and $m \ge {\rm poly}(n, \lambda_0^{-1}, \delta^{-1})$. Consider any loss function $\ell: \mathbb{R} \times \mathbb{R} \rightarrow [0,1]$. Then with probability at least $1-\delta$ over the random initialization and training samples, the network trained by gradient descent for $T \ge \Omega(\frac{1}{\eta \lambda_0} \log \frac{n}{\delta})$ iterations has population risk $R_{\mathcal{D}}(Q)$ that is bounded as follows:
\small{
\begin{equation}
\begin{aligned}
   R_{\mathcal{D}} & \le  \sqrt{\frac{  {\bf Y}^\top (\boldsymbol{\Theta}+ {\lambda}/{\sigma^2_0} {\bf I})^{-1} {\bf Y}}{n}} + \frac{\lambda}{\sigma^2_0}  \sqrt{ \frac{ {\bf Y}^\top(\boldsymbol{\Theta}+  {\lambda}/{\sigma^2_0} {\bf I})^{-2} {\bf Y}} {n}}  +O \bigg(\sqrt{\frac{\log \frac{n}{\lambda_0 \delta}}{n}} \bigg).
   \end{aligned}
\end{equation}}
\end{theorem}

\begin{proof}
In this proof, we use Rademacher-complexity analysis. Let $\mathcal{H}$ be the reproducing kernel Hilbert space (RKHS) corresponding to the kernel $k(\cdot, \cdot)$. It is known that the RKHS norm of a function $f_{\rm ntk}({\bf x}) = \boldsymbol{\Theta}^\infty({\bf x},{\bf X}) \big (\boldsymbol{\Theta}({\bf X},{\bf X}) + \lambda/{\sigma^2_0} {\bf I} \big)^{-1} {\bf Y} = \boldsymbol{\alpha}^\top k({\bf x},{\bf X})$ is $\| f_{\rm ntk}\|_\mathcal{H} = \sqrt{\boldsymbol{\alpha}^\top k({\bf X},{\bf X}) \boldsymbol{\alpha}} $, where $ k = \boldsymbol{\Theta}$ and $\boldsymbol{\alpha} = {\lambda}/{\sigma^2_0} {\bf I})^{-1} {\bf Y}  $. Then we can bound the $\| f_{\rm ntk}\|_\mathcal{H}$.


$$
\| f_{\rm ntk}\|_\mathcal{H} = \sqrt{ {\bf Y}^\top (\boldsymbol{\Theta}({\bf X},{\bf X})+ {\lambda}/{\sigma^2_0} {\bf I})^{-1} \boldsymbol{\Theta}({\bf X},{\bf X})(\boldsymbol{\Theta}({\bf X},{\bf X})+ {\lambda}/{\sigma^2_0} {\bf I})^{-1} {\bf Y} } \le \sqrt{{\bf Y}^\top (\boldsymbol{\Theta}({\bf X},{\bf X})+ {\lambda}/{\sigma^2_0} {\bf I})^{-1} {\bf Y}} 
$$

 For function class $\mathcal{F}_B = \{f({\bf x}) = \boldsymbol{\alpha}^\top k({\bf x},{\bf X}): \| f\|_\mathcal{H} \le B \}$, it is shown that its empirical Rademacher complexity can be bounded as,
$$
 {\widehat R}_S(\mathcal{F}_B) = \frac{1}{n} \mathbb{E}\big[\sup_{f \in \mathcal{F}_B} \sum_{i=1}^n f({\bf x}_i) \gamma_i \big] \le \frac{B \sqrt{{\rm Tr}[k({\bf X},{\bf X})]}}{n}
$$

Assume that $ {\rm Tr}[k({\bf X},{\bf X})] \approx n$.
Recall the standard generalization bound from Rademacher complexity, with probability at least $1 - \delta$, we have,
$$
\sup_{f \in \mathcal{F}} [ \mathbb{E}_\mathcal{D}[\ell(f({\bf x}),y)] -\frac{1}{n} \sum_{i=1}^n \ell (f({\bf x}_i),y_i) ] \le 2 {\widehat R}_S(\mathcal{F}) + 3\sqrt{\frac{\log(2/\delta)}{2n}}
$$
Thus we have,
$$
\begin{aligned}
   R_{\mathcal{D}} & \le \sqrt{\frac{  {\bf Y}^\top (\boldsymbol{\Theta}+ {\lambda}/{\sigma^2_0} {\bf I})^{-1} {\bf Y}}{n}} + \frac{\lambda}{\sigma^2_0}  \sqrt{ \frac{ {\bf Y}^\top(\boldsymbol{\Theta}+  {\lambda}/{\sigma^2_0} {\bf I})^{-2} {\bf Y}} {n}}+O \bigg(\sqrt{\frac{\log \frac{n}{\lambda_0 \delta}}{n}} \bigg).
   \end{aligned}
$$
\end{proof}

\section{Additional Experiments}

This section contains additional experimental results.

\subsection{Training dynamics of wide probabilistic network}

We conduct an experiment similar to the experiment conducted in section \ref{tdwp} but without incorporating the KL term in the objective function. The same results can be observed, further verifying our theoretical results. 

\subsection{Comparison of gradient norm with respect to $\boldsymbol{\mu}$ and $\boldsymbol{\sigma}$} \label{sec:norm}

In Lemma \ref{lem:4.9}, we assume that the variance weight $\boldsymbol{\sigma}$ is fixed. To verify if this assumption is reasonable in the practical training process, we conduct an experiment, to compare the gradient of norm with respect to $\boldsymbol{\mu}$ and $\boldsymbol{\sigma}$. The result is shown in Figure \ref{fig:12}. We can see that the gradient norm  of $\nabla_{\boldsymbol{\mu}} f $ is much larger than that of $\nabla_{\boldsymbol{\sigma}} f $, which implies that $\boldsymbol{\sigma}$ is effectively fixed during gradient descent training.


\begin{figure}[t!]
\centering
\subfloat{\includegraphics[width=0.3\textwidth]{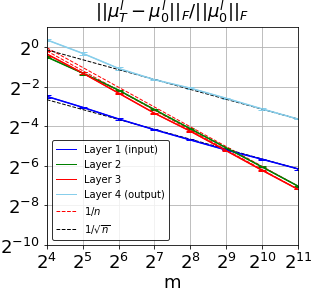}} 
\renewcommand{\thefigure}{S1}
\vspace*{-4mm}
\caption{Relative Frobenius norm change in $\mu$ during training with MSE loss that does not include the KL term, where $m$ is the width of the network.} 
\label{fig:11} 
\end{figure}

\begin{figure}[t!]
\centering
\subfloat{\includegraphics[width=0.32\textwidth]{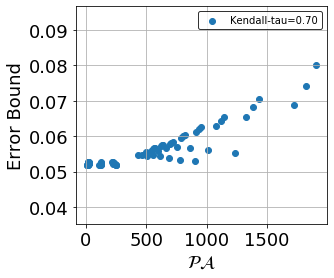}}
\renewcommand{\thefigure}{S2}
\vspace*{-4mm}
\caption{Correlation between aggregated proxy $\mathcal{PA}$ and generalization bound.} 
\label{fig:13} 
\end{figure}

\begin{figure}[t!]
\centering
\subfloat{\includegraphics[width=0.3\textwidth]{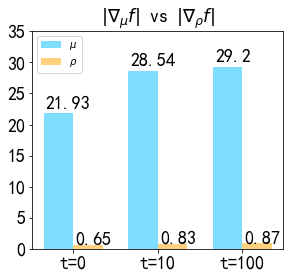}}
\renewcommand{\thefigure}{S3}
\vspace*{-4mm}
\caption{Comparison between the gradient of mean $\boldsymbol{\mu}$ and standard deviations $\boldsymbol{\sigma}$.} 
\label{fig:12} 
\end{figure}
\subsection{Correlation between generalization bound proxy metric and generalization bound}
In Figure \ref{fig:3}, we observe a positive and significant correlation between $\mathcal{PA}$ and generalization bound held among different values of a selected hyperparameter while fixing other hyperparameters. Furthermore, we provide a Figure \ref{fig:13} presenting the correlation for aggregated values of $ \rho_0$ and $ \lambda$, under the circumstance where 50\% data is used for prior training. We can clearly see that lower $\mathcal{PA}$ 
corresponds to the lower bound, with a strong positive Kendall-tau correlation of 0.7.

\subsection{Grid search}
For selecting hyperparameters, we conduct a grid search over $\rho_0$, percent of prior data, and KL penalty $\lambda$. Notably, we do grid sweep over the data for prior training with different proportion in [0.2, 0.3, 0.4, 0.5, 0.6, 0.7, 0.8, 0.9] since 0.2 is the minimum proportion required for obtaining a reasonably lower value generalization bound \citep{dziugaite2020role}. For the rest, we run over $\rho_0$ at value [0.03, 0.05, 0.07, 0.09, 0.1, 0.3, 0.5, 0.7] for FCN ([0.05, 0.07, 0.09, 0.1, 0.3, 0.5, 0.7, 0.9] for CNN) and KL penalty at [0.0001, 0.0005, 0.001, 0.005, 0.01, 0.05, 0.1, 0.5, 1] for both structures.



\end{document}